\documentclass{article} 
\usepackage{paper,times}

\usepackage{amsmath,amsfonts,bm}

\def\eqref#1{equation~\ref{#1}}

\def\1{\bm{1}}

\DeclareMathAlphabet{\mathsfit}{\encodingdefault}{\sfdefault}{m}{sl}
\SetMathAlphabet{\mathsfit}{bold}{\encodingdefault}{\sfdefault}{bx}{n}

\usepackage{hyperref}
\usepackage{url}
\usepackage{booktabs}
\usepackage{amsmath}
\usepackage{graphicx}
\usepackage{multirow}
\usepackage{amssymb}
\usepackage{amsthm}

\newtheorem{theorem}{Theorem}
\newtheorem{lemma}{Lemma}

\usepackage{pgfplots}
\pgfplotsset{compat=1.18} 
\usepgfplotslibrary{groupplots}
\usepackage{wrapfig}
\usepackage{float}

\title{Masked Topology Modeling for Self-Supervised Learning on Parametric CAD}

\iclrfinalcopy
\author{\textbf{Heinrich Jiang} \\
  StoryGold AI \\
  \texttt{heinrich@storygold.com} \\
  \And
  \textbf{Jennifer Jang} \\
  StoryGold AI \\
  \texttt{jennifer@storygold.com} \\
}

\begin{document}

\maketitle

\begin{abstract}
Computer aided design (CAD) is ubiquitous: virtually any modern object was designed using editable CAD tools. However, with the shortage of available CAD datasets in its native editable and parametric format, boundary representation (B-Rep), it is ever more important to develop data-efficient methods for this domain.

We present a new self-supervised pretraining task \emph{Masked
Topology Modeling} (MTM) that leverages the face-adjacency graph, an induced structure unique to B-reps that the encoder can be asked to reconstruct. MTM masks a fraction of
edges and trains a small head to predict each masked edge's convexity and curve type from the encoder's post-message-passing face features.

We combine MTM with a MoCo-style momentum-queue contrastive learning over B-rep-aware augmentations, a BFS-connected face-region masked-reconstruction objective, and pretraining on the ABC dataset and our new procedurally generated dataset to show strong performance on a number of benchmarks.
\end{abstract}

\section{Introduction}

Computer-Aided Design (CAD) is a foundational technology that has  revolutionized manufacturing \citep{groover1983cad}, modern engineering, construction \citep{heesom2004trends}, architecture \citep{szalapaj2013cad}, healthcare \citep{fujita2008computer}, and product development \citep{liu2021fast}. By replacing time-consuming and error-prone manual drafting, its ability to precisely build complex 3D models allows designers to visualize, simulate, and test products before being sent for manufacturing or construction. 

Boundary Representation (B-rep) is the standard format used by modern CAD systems to define and edit solid 3D models. Unlike mesh formats which approximate a shape using a collection of flat polygons \citep{bernardini1999automatic}, B-rep defines an object by its exact, mathematical boundaries. This "watertight" model is composed of two core components: {\it geometry}, which provides the precise mathematical description (such as planes, cylinders, or complex NURBS curves) for where those elements exist in space, and {\it topology}, which is the collection of faces, edges, and vertices and the logical relationships and constraints that connect them. This dual structure creates an unambiguous, data-rich solid that is mathematically precise \citep{faux1979computational}. 

Despite being so popular and widespread, there is a lack of publicly available B-rep data. In practically every industry, CAD data is considered critical intellectual property necessary for competitive advantage \citep{stjepandic2015intellectual}. Sharing such files can expose companies to significant additional risks, including physical sabotage, supply chain poisoning \citep{karahan2025unveiling}, and data integrity issues in distributed manufacturing \citep{li2025enhancing}.
Converting to non-parametric formats such as 3D meshes (e.g. STL) destroys the design intent and feature tree history, making it very difficult to reverse engineer the original design \citep{lecallard2021mesh}. In practice, CAD designers typically send manufacturers only the mesh as a safeguard from theft or modification of the core design, since the mesh only provides an approximation of the objects \citep{kantaros2024intellectual}.


Recent AI systems have shown profound success across various industries. AI typically depends on large, high-quality datasets, but CAD as a use case lacks comparable access to such data. Industrial CAD data is scarce, expensive to label, and often proprietary, making it ever more critical to study data-efficient methods for CAD. In other domains, self-supervised learning (SSL) has substantially improved label efficiency as shown by contrastive and masked-autoencoding approaches \citep{chen2020big,he2022masked,bahri2021scarf}. It turns out that B-reps provide a particularly rich structure for SSL when viewed as topological graphs because they expose intrinsic geometric and topological features arising from the signals—face adjacency, edge topology, curve and surface types, continuity, and geometric constraints—that can be learned without manual labels. Therefore, although B-rep data is limited and labeled B-rep data is even more limited, each B-rep nonetheless contains rich internal structure that can generate many high quality self-supervised training signals. Therefore, SSL is a natural way to improve both data and label efficiency.

Our contributions are as follows:

\begin{enumerate}
\item \textbf{Masked Topology Modeling (MTM).} An SSL objective that masks the induced face-adjacency graph edges and predicts each masked edge's convexity and curve type from post-message-passing face features.
To our knowledge this is the first use of \emph{masked graph-topology
prediction} as an SSL objective for B-rep and the labels are extracted for free by the geometry kernel.
\item \textbf{Theoretical Analysis.} We show identifiability results about MTM's ability to recover the edge's convexity and curve type, and a separability result that shows MTM forces the encoder to represent geometry that face-only approaches ignore.
\item \textbf{End-to-end SSL procedure.} We combine MTM with MoCo momentum-queue contrastive learning \citep{moco} and a breadth first search (BFS)-connected
face-region masked reconstruction objective and train with a $27.9$M-parameter encoder based on UV-grid \citep{jayaraman2021uv}, GATv2 \citep{brody2021attentive}, and Transformers \citep{vaswani2017attention}.
\item \textbf{Synthetic SSL B-rep dataset.} We adapt a procedural sketch-extrude generator \citep{cadrecode} to emit B-rep solids with topology supervision. We pretrain our method with ABC dataset \cite{abc} and the synthetic dataset. Ablations show the synthetic data during pretraining contributes significant gains on downstream real-world tasks.
\item \textbf{Strong empirical results} across four downstream benchmark datasets (F360, SolidLetters, MFInstSeg and CadSynth). We show SOTA performance on F360, SolidLetters, and MFInstSeg, and competitive performance on CadSynth with recent SOTA in-distribution methods in the low label budget setting.
\end{enumerate}

\section{Related Work}

{\bf B-rep encoders}. Research on B-rep has received a lot of recent attention due to its ability to effectively represent 3D models \citep{guo2022complexgen,jayaraman2022solidgen,willis2022joinable,cherenkova2024spelsnet,xu2025autobrep}. We begin by highlighting the early work in B-rep representation learning. UV-Net \citep{jayaraman2021uv} and BRepNet \citep{lambourne2021brepnet} are two foundational B-rep encoders. UV-Net represents each face as a UV-grid processed by a CNN and the face-graph by graph convolutions, producing per-face embeddings. BRepNet convolves directly over oriented coedges of the data structure. Both are supervised architectures, trained end-to-end for a downstream head. Relative to BRepNet, our work differs in the role of topology. BRepNet conditions message passing on it, whereas we hide part of it and make it a prediction target. We adopt the UV-Net architecture as our backbone but decouple representation learning from supervision.
 CADOps-Net \citep{dupont2022cadops} learns a representation via joint learning of CAD operation types and steps. Hierarchical CADNet \citep{colligan2022hierarchical} proposes a hierarchical B-Rep graph representation which encodes information about the surface geometry and face topology trained in a supervised way. AAGNet \citep{wu2024aagnet} is a multi-task graph network over face-adjacency graphs, and BRepMFR \citep{zhang2024brepmfr} uses domain adaptation.

{\bf Self-supervised learning for B-reps}. The UV-Net paper also proposes a self-supervised contrastive variant, which pretrains the encoder by pulling together augmented views of the same solid. We use contrastive learning as one component but add a structural, non-contrastive objective (MTM) on top. BRep-BERT \citep{lou2023brep} adapts masked language modeling by masking and reconstructing per-node (face) attributes. We differ in the masking unit and target: we mask edges and predict relational, kernel-computed properties (convexity and curve type) rather than node features. 
GC-CAD \citep{quan2024self} and ContrastCAD \citep{jung2024contrastcad} learn through contrastive alignment of augmented B-rep views. As with UV-Net's contrastive variant, the learning signal is view-invariant. MTM adds local topological reasoning that these objectives do not explicitly supervise. Most recently, BRepMAE \citep{yao2026brepmae} takes a generative masked-autoencoding route, reconstructing masked geometry. Ours is discriminative: classifying a small set of topological relations at hidden edges rather than regenerating a region. Masked B-rep Autoencoder \citep{li2026masked} masks and reconstructs node geometry and attributes through a decoder. MTM masks graph topology and recovers it as a discriminative edge-classification task (convexity and curve type), coupled with a global contrastive objective and no reconstruction decoder.

\section{Method}

\subsection{Viewing a B-rep as a Labeled Topological Graph}
\label{sec:brep-as-graph}

\begin{wrapfigure}{r}{0.5\linewidth}
  \centering
  \includegraphics[width=\linewidth]{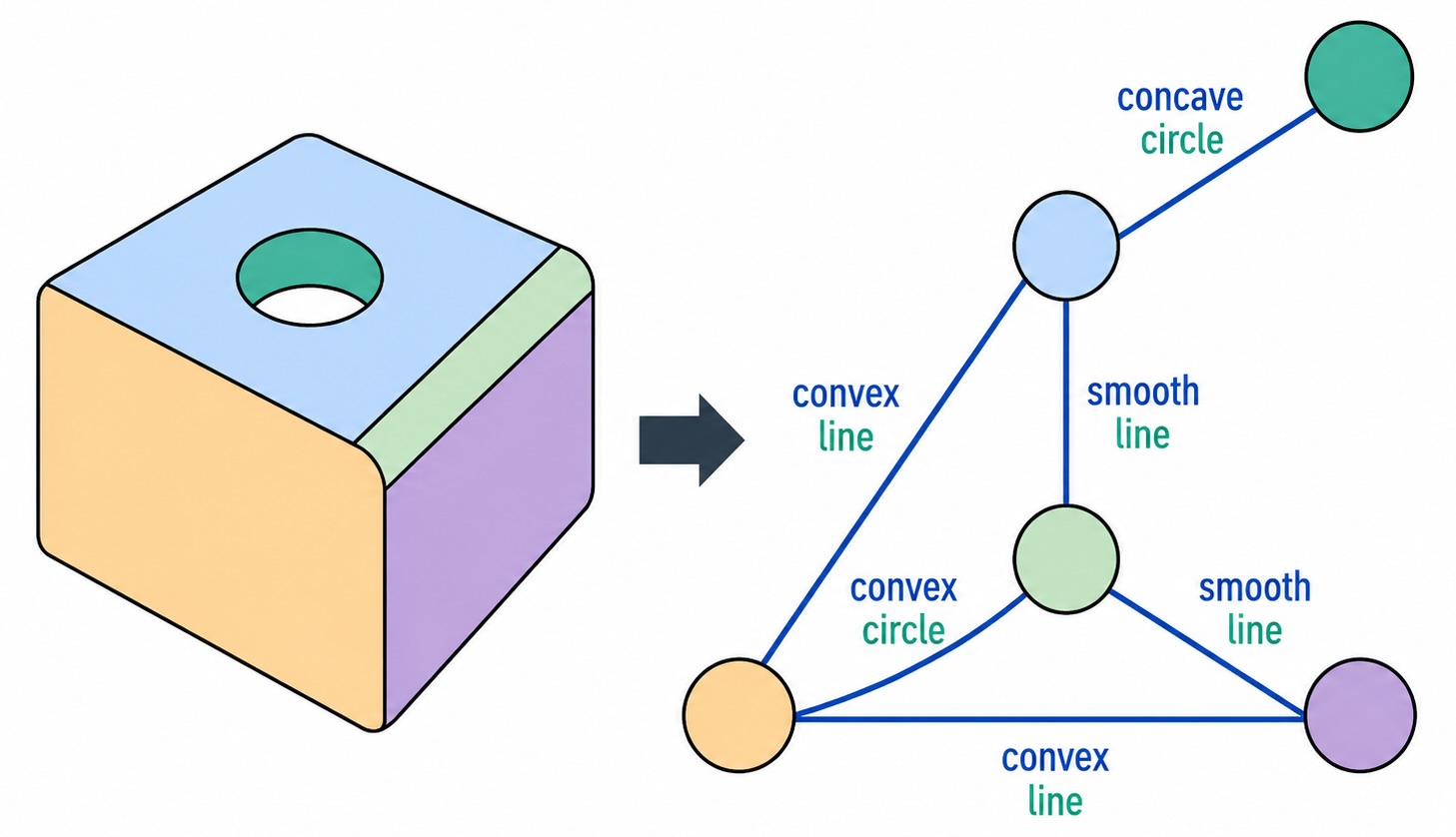}
  \caption{\textbf{A B-rep solid as an labeled face-adjacency graph.}
    Each face of the solid (left) becomes a node (right), and two
    nodes are joined by a graph edge iff their faces meet along a shared B-rep edge. Every graph edge carries the two labels MTM predicts: a convexity class
    (\emph{convex}, \emph{concave}, \emph{smooth}, \emph{knife}) and a curve type
    (\emph{line}, \emph{circle}, \emph{B-spline}, \emph{ellipse}, \emph{other}),
    both computed deterministically by the geometry kernel. Because this
    topological signal lives on the graph's edges, it can be masked and predicted directly.}
  \label{fig:brep-graph}
\end{wrapfigure}

Unlike a mesh or point cloud, the B-rep representation contains the \emph{topology} of the structure: a hierarchy of
faces, edges, and vertices together with their incidence relations (i.e. which edge
bounds which faces, which vertices terminate which edge, and orientation via
co-edges/half-edges). It's a combinatorial structure computed and guaranteed
consistent by the CAD kernel. 
MTM operates on a specific induced structure from the B-rep, the \emph{face-adjacency graph}.

The face-adjacency graph \citep{ansaldi1985geometric} was originally introduced to expose the topology of a solid model.
The nodes are faces and the edges are shared B-rep edges. 
We take each face of the solid to be a node. The node's input feature is the geometry the UV-Net encoder already extracts: a UV-grid sampling of the
trimmed surface (points, normals, and a trimming/visibility mask) passed through the surface CNN to a per-face embedding. A node is thus a patch of surface together with its local shape. Two faces $f_i$ and $f_j$ are connected by a graph edge $e_{ij}$ if and only if they share a B-rep edge, i.e. they meet along a common curve.  

We can further automatically label each graph edge with the following two topological attributes: {\it convexity} $c(e)\in\{\mathrm{concave},\mathrm{convex},
  \mathrm{smooth},\mathrm{knife}\}$, the dihedral relationship of the two
  faces across the seam. If the material folds inward, then it's {\it concave}; and if it's outwards, then it's {\it convex}. If they meet tangentially, that is at approximately $180$ degree angle (the boundary between concave and convex), then it's {\it smooth}. Finally, the degenerate case where it meets at approximately $0$ degree angle, then it's classified as {\it knife}. The second attribute is {\it curve type} $t(e)\in\{\mathrm{line},\mathrm{circle},
  \mathrm{B\text{-}spline},\mathrm{ellipse},\mathrm{other}\}$—the geometry of
  the seam curve itself.
Both are computed deterministically by the geometry kernel during preprocessing
(signed dihedral angle for convexity, the parametric curve tag for type—see
Appendix~\ref{app:preproc}), so these labels are \emph{free}. No human annotates
them. This is the reason topology can serve as a self-supervised signal. The
graph edges come pre-labelled by construction. Labeling the edges of the face-adjacency graph is an old idea used for process planning \citep{joshi1988graph}. However, it was a binary concave/convex label. Meanwhile, we significantly expand the labels for MTM, leading to higher quality self-supervised training signals.


\subsection{Contrastive Learning}

Our encoder follows UV-Net~\citep{jayaraman2021uv}: each face is sampled on a fixed $U\times V$ parametric grid and embedded by a surface CNN, each edge by a curve CNN, and the resulting per-face features are refined by message passing over the face-adjacency graph to yield embeddings $h_i\in\mathbb{R}^{384}$. We use this architecture unchanged except that we replace UV-Net's supervised training with our self-supervised objective. Full sampling, layer, and dimension details are in Appendix~\ref{app:backbone}.

\begin{figure}[t]
  \centering
  \includegraphics[width=0.8\linewidth]{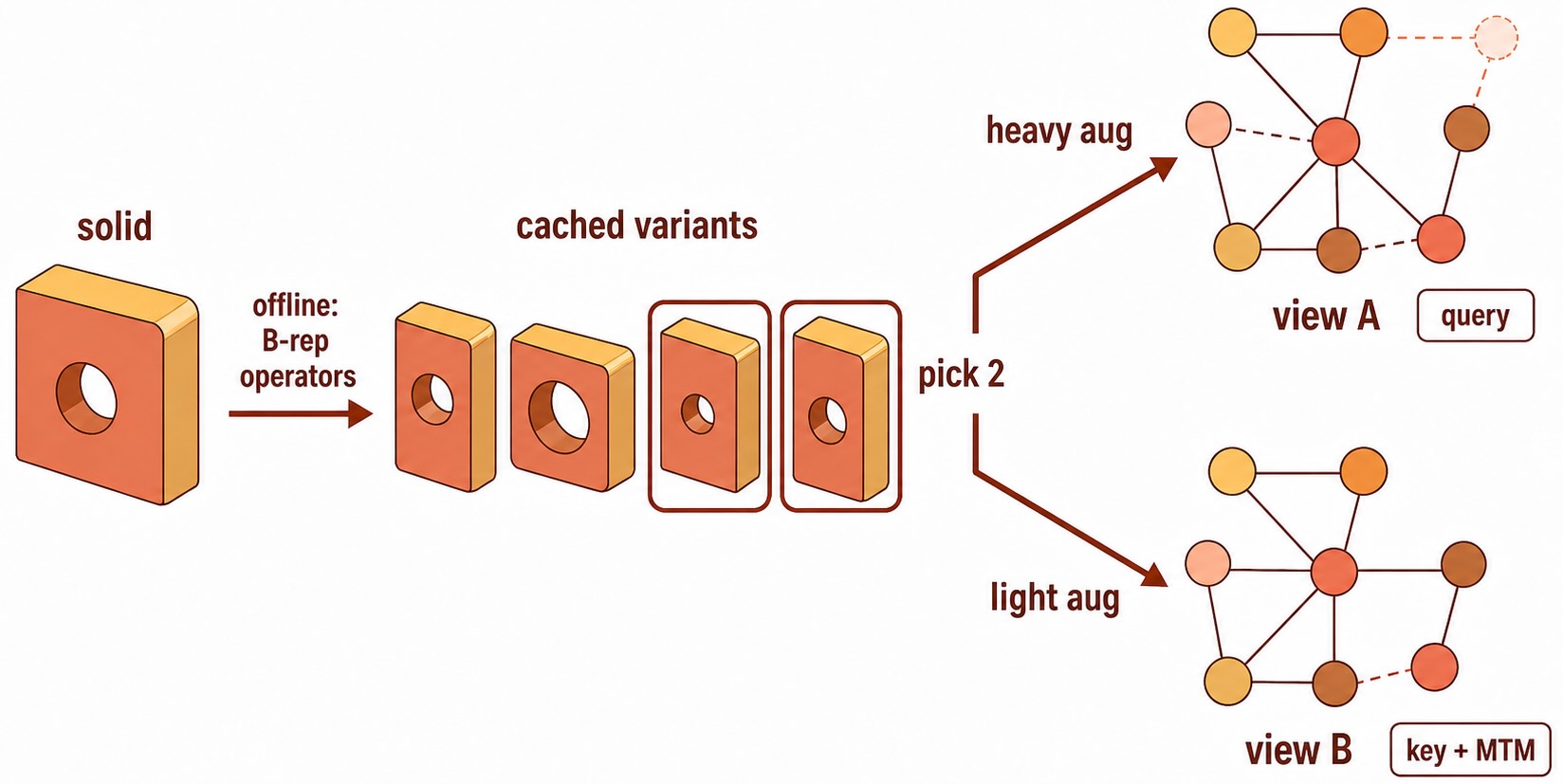}
  \caption{\textbf{Two-level positive-pair construction.} \emph{Offline}, each
  source solid is perturbed by B-rep operators into several cached variants that
  differ in geometry. \emph{Online}, a training pair draws two of these variants
  and augments them asymmetrically at the graph level. One receives heavy
  augmentation to form view~A (the contrastive query), the other light
  augmentation to form view~B (the contrastive key, and the graph on which MTM
  and region reconstruction operate). The pair therefore differs in both geometry
  (distinct perturbations) and graph structure (heavy vs. light corruption),
  while view~B stays lightly augmented so the auxiliary reconstruction targets
  remain well-posed.}
  \label{fig:augment}
\end{figure}

We adopt Momentum Contrast (MoCo) \citep{moco}, which casts the objective as a dictionary-lookup task. Each input is rendered into two augmented views. View A is passed through the online (query) encoder to produce a query, while View B is passed through a momentum key encoder to produce a key. The query is trained to match its corresponding key (the positive) against a large set of negatives drawn from a large queue of negatives (here $16{,}384$ negatives). This queue decouples the number of negatives from the batch size and is optimized with an InfoNCE loss \citep{oord2018representation}. Full details are in Appendix~\ref{app:contrastive}.

Positive pairs are formed at two levels. \emph{Offline}: for each source
solid we precompute several perturbed variants by applying a randomly chosen
B-rep operator through the OpenCASCADE kernel, each validated for a watertight solid. A training pair draws two different variants of the same source. \emph{Online}: each drawn view further undergoes face drop, edge drop, connected-subgraph sampling, and coordinate jitter. Augmentation is
asymmetric: view~A (the contrastive query) receives the full heavy
augmentation, while view B receives only light augmentation. View B serves both as the contrastive key and as the graph on which the auxiliary objectives (MTM and region reconstruction) operate. Keeping B lightly augmented keeps those
reconstruction targets well-posed, because both auxiliary objectives read
from the same view B encoding. 

{\bf Pretraining data.} Our backbone is pretrained entirely out-of-domain, from two sources. The first is real B-rep solids (in STEP format) from ABC \citep{abc}
(${\sim}1$M shapes), the largest public collection of CAD models. We apply a filtering and deduplication pass to discard degenerate or malformed solids and remove near-duplicate models. The second is a new procedurally generated dataset: we adapt the CAD-generation procedure of CAD-Recode \citep{cadrecode}, which samples random 2D profiles, places them on random planes, extrudes them into 3D primitives, and combines these into a single solid. We add a topology-preserving parameter-jitter operation that perturbs a solid's continuous parameters (extrusion depths, radii, positions) while holding its face-edge topology fixed, producing semantically equivalent variants of a shape. Together with our B-rep operator augmentations, this supplies the positive pairs for contrastive learning entirely from synthetic data (see Appendix~\ref{app:synth}).

Of the ABC dataset, for the convexity label, approximately $39\%$ of the labels were convex, $39\%$ concave, $22\%$ smooth, and around $0.1\%$ were knife. For the curve type label: $56\%$ line,  $25\%$ circle, $17\%$ B-spline, $2\%$ ellipse, and very few were labeled other, which is meant to capture degenerate cases.

\subsection{Masked Topology Modeling (MTM)}
\label{sec:mtm}

\begin{figure}[t]
  \centering
  \includegraphics[width=0.9\linewidth]{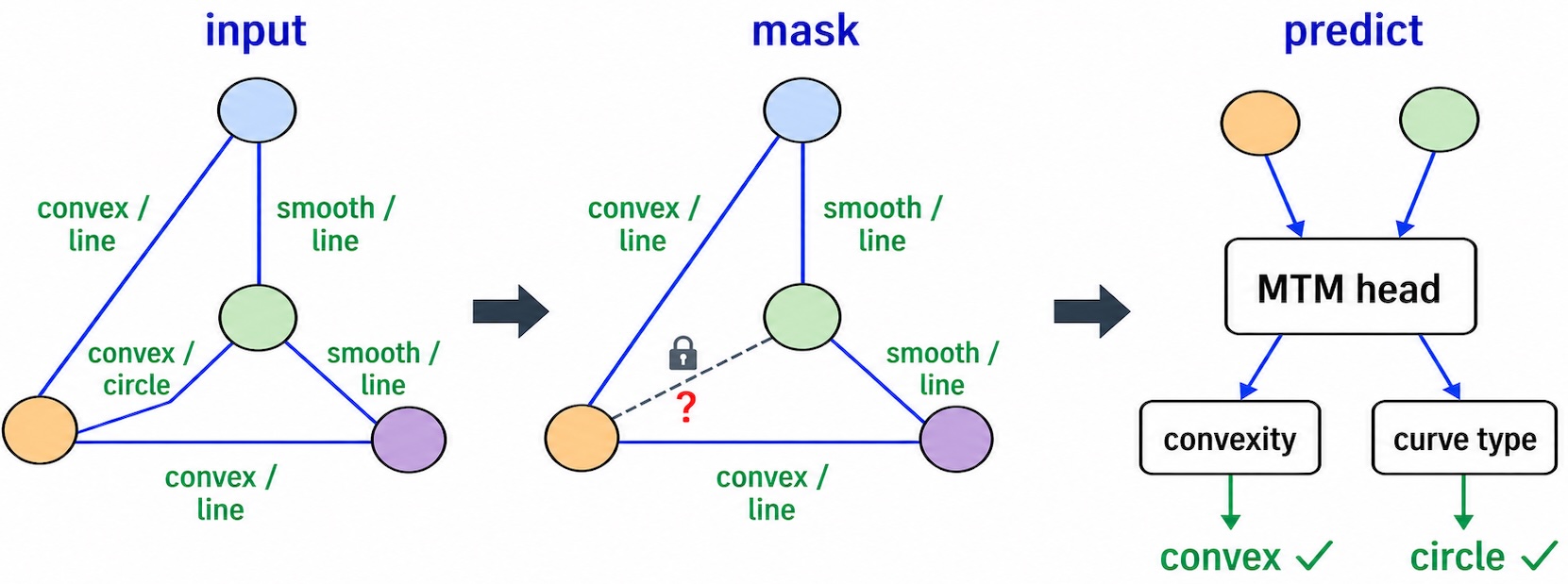}
  \caption{\textbf{Masked Topology Modeling (MTM).} Left: each edge of the
  face-adjacency graph carries kernel-computed labels --- a convexity class and a
  curve type. Middle: we hide one edge, deleting it and its labels from the graph
  before encoding, so no message crosses that seam. Right: from the two endpoint
  faces alone (each contributing only its own boundary evidence, never the
  removed edge) an MTM head recovers the hidden convexity and curve-type
  labels.}
  \label{fig:mtm}
\end{figure}

MTM is a \emph{masked prediction} objective on the face-graph's edge set. The key
design choice is that the encoder never sees the masked edges. They are
deleted from the graph before message passing so the head must reconstruct
each missing edge's semantics purely from the geometry of its two endpoint
faces, as propagated through the remaining topology. Given a shape $S{=}(F,E)$:

\begin{enumerate}
\item \textbf{Mask.} On view B, sample a masked edge set
$M\subset E$ at rate $p_{\text{mask}}{=}0.7$ and delete it from the edge index, so
the encoder receives only $(F,E\setminus M)$. We record each masked edge's two
endpoint faces and its ground-truth $(\text{convexity},\text{curve-type})$ labels.
\item \textbf{Forward.} The encoder $\phi$ runs message passing over the
\emph{masked} graph to produce per-face features $h_i$. Because the masked edges
are absent, $h_i$ aggregates only from a face's \emph{remaining} neighbors (i.e. the
head cannot see the edge it must predict).
\item \textbf{Predict.} For each masked edge $e{=}(i,j)$, the ordered endpoint
feature $[h_i;h_j]$ is passed to two lightweight MLP heads predicting convexity
($4$ classes) and curve type ($5$ classes).
\item \textbf{Loss.}
\begin{equation}
L_{\text{topo}} = \tfrac{1}{2}\!\left(
\mathrm{CE}(\mathrm{MLP}_{\text{convexity}}([h_i;h_j]),\,c(e)) +
\mathrm{CE}(\mathrm{MLP}_{\text{curve-type}}([h_i;h_j]),\,t(e))\right),
\end{equation}
averaged over masked edges, weight $\lambda_{\text{topo}}{=}0.5$ and CE denotes cross-entropy loss. Both labels are extracted deterministically and no labels are required. Details are in
Appendix~\ref{app:mtm-details}.
\end{enumerate}

The heads are lightweight (${\approx}0.1$M params, discarded after pretraining)
and add negligible compute. They operate on the light-augmentation view's face features that the region-reconstruction objective already compute. Note that the edge features are deliberately \emph{not} conditioned on the edge's labels, because those labels are exactly what MTM predicts. Supplying them as input would leak the target.


\subsection{Face-region masked reconstruction}
\label{sec:region}
Complementing MTM's edge-level objective, we add a BFS-connected face-region masked reconstruction objective on view B. We mask a single connected
blob of faces (replacing their UV-grids with a learnable mask token) and, from each masked face's feature, reconstruct its geometry (MSE) and its coarse
geometric type (cross-entropy), giving
\begin{equation}
L_{\text{region}}=\mathrm{MSE}(\text{geom})+\mathrm{CE}(\text{type}),
\qquad \lambda_{\text{region}}{=}1.0.
\end{equation}
Masking a \emph{connected} region (rather than scattered faces) forces the encoder
to infer a whole face from its surroundings.
Full specification in Appendix~\ref{app:region}.

\subsection{Combined objective}
\label{sec:combined-obj}
\begin{equation}
L = L_{\text{NCE}}(z_a, z_b) + \lambda_{\text{topo}}\, L_{\text{topo}}
    + \lambda_{\text{region}}\, L_{\text{region}}.
\end{equation}
The headline backbone is a $30$-epoch continuation from an ABC+synth-pretrained
checkpoint with $p_{\text{mask}}{=}0.7$, $\lambda_{\text{topo}}{=}0.5$, and region
masking ($\lambda_{\text{region}}{=}1.0$, region ratio $0.4$).

\section{Theory}
\label{sec:theory}

We present two results concerning the masked topology objective: one showing that its two prediction targets differ in how learnable they are (identifiability), and another showing that it compels the encoder to capture geometric information that face-level objectives are shown to miss entirely (separability).

We view a solid $S$ as a \emph{face-adjacency graph} $G(S)=(V,E)$. Each node is a face (a smooth surface patch) and each edge $e=\{i,j\}$ records that faces $i,j$ meet along a shared boundary curve $\sigma_e$. MTM supervises two per-edge targets: the \emph{convexity}
$c(e)\in\{\text{convex},\text{concave},\text{smooth},\text{knife}\}$, defined
relative to the material (interior) side; and the \emph{curve type}
$t(e)\in\{\text{line},\text{circle},\text{ellipse},\text{B-spline},\text{other}\}$. Let $\Phi$ map a graph to per-face embeddings $\Phi(G)_v\in\mathbb{R}^d$, and let $\mathcal{M}\subseteq E$ be a random per-shape edge subset whose connections, features, and labels are removed and held out as targets. Both directed half-edges of a masked edge are removed together, so no masked label survives in the input. Each $h_v=\Phi(G\setminus\mathcal{M})_v$ aggregates only surviving structure, and the head predicts $(c(e),t(e))$ from $(h_i,h_j)$. Write $H(Y\mid X)$ for the residual uncertainty about $Y$ given $X$ ($=0$ iff $X$ determines $Y$). All proofs are in Appendix~\ref{app:theory}.

The next result says the two labels MTM asks the encoder to predict have different ceilings. Convexity is always learnable to perfect accuracy, but curve type is not, because different curves can look identical at finite
sampling resolution. For instance, a spline can pass through the same sample points as a true circular arc while matching it nowhere else, making the two curve types impossible to tell apart from those samples alone. Writing $\mathcal{S}_e$ for the finite samples an encoder sees near $e$:

\begin{theorem}[Identifiability]
\label{thm:ordering}
Assume {(A1)} each face is $C^1$ (continuously differentiable) near $\sigma_e$; {(A2)} the faces meet
along exactly one curve; {(A3)} the encoder is finite-resolution---each
$h_v$ comes from finitely many position-and-normal samples resolving the normal near $\sigma_e$ (matching real encoders, which use a fixed UV grid per face and a fixed number of samples per curve); and {(A4)} with positive probability an edge carries a curve whose type is not determined by its samples $\mathcal{S}_e$ (e.g.\ a spline matching a conic at the sampled points). Then, the smallest achievable MTM loss is $0$ for convexity but at least $\mathbb{E}[H(t(e)\mid\mathcal{S}_e)]>0$ for curve type. This lower-bounds \emph{any} finite-resolution encoder. The convexity head is thus trainable to strictly lower loss than the type head, with residual type error concentrated on curves matching a simpler analytic form at the sampled resolution (near circular ellipses, large-radius arcs, splines that locally mimic conics).
\end{theorem}

Write $\gamma(f)$ for the \emph{intrinsic geometry} of face $f$ (its surface patch up to rigid motion). Call a representation \emph{face-level} if its value is a function of $(\gamma,G)$ alone. That is, it sees each face's intrinsic shape and the face-adjacency graph stripped of all geometry information, and in particular never sees the \emph{relative orientation} of adjacent faces. Masked-face reconstruction and face-level contrastive learning are of this kind. MTM, supervising $c$, is not.

The next result shows the existence two different solids that look identical to any method judging faces in isolation (i.e. same face shapes, same adjacency graph) yet differ in how the faces fold together (one has a convex crease, the other concave). A face-level method will give them the same representation and never tell them apart. MTM will not, because to predict the opposite creases it must encode the fold. MTM is therefore strictly more discriminative.

\begin{theorem}[Separability]
\label{thm:separation}
There exist closed solids $S,S'$ and an adjacency-preserving bijection
$\phi\colon V\!\to\!V'$ ($G(S)\!\cong\!G(S')$) with $\gamma(f)=\gamma(\phi(f))$ for
every face, yet $c(e)\neq c(\phi(e))$ for some edge $e$. Then \textup{(i)} every
objective depending on $(\gamma,G)$ alone admits a global optimum with
$h_f(S)=h_{\phi(f)}(S')$ for all $f$, on which $S,S'$ are indistinguishable; while
\textup{(ii)} no $(\gamma,G)$-measurable map is a sufficient statistic for $c$, so
any representation attaining zero MTM loss must emit $c(e)\neq c(\phi(e))$ from the endpoint embeddings, giving at least one endpoint a distinct embedding and
separating $S,S'$. Hence MTM-optimal representations are strictly more
discriminative than any face-level one: MTM must separate a pair that every
face-level optimum is free to merge.
\end{theorem}
\section{Experiments}

\begin{table}[t]
\centering
\setlength{\tabcolsep}{4pt}
\caption{Fusion~360 Gallery segmentation, full-data and few-shot (8-way
support, evaluated on the full test set). Our full-data number is the 3-seed
mean over downstream fine-tuning seeds. Reported as Top-1~/~mIoU (macro).}
\label{tab:f360}
\begin{tabular}{llcccccc}
\toprule
& & \multicolumn{2}{c}{Full-data} & \multicolumn{2}{c}{10-shot} & \multicolumn{2}{c}{20-shot} \\
\cmidrule(lr){3-4}\cmidrule(lr){5-6}\cmidrule(lr){7-8}
Method & Pretrain data & Top-1 & mIoU & Top-1 & mIoU & Top-1 & mIoU \\
\midrule
UV-Net & SolidLetters / ABC & 92.30 & 72.40 & -- & -- & -- & -- \\
BRepNet & F360 (supervised) & 94.30 & 81.40 & -- & -- & -- & -- \\
BRep-BERT & F360 (transductive) & 95.14 & 82.88 & 60.55 & 20.80 & 63.61 & 25.41 \\
From-scratch & None & -- & -- & 60.42 & 20.28 & 62.71 & 27.12 \\
\textbf{Ours} & ABC + synthetic & \textbf{95.6} & \textbf{84.55} & \textbf{89.65} & \textbf{66.58} & \textbf{91.13} & \textbf{68.92} \\
\bottomrule
\end{tabular}
\end{table}

\begin{table}[t]
\centering
\caption{SolidLetters, full-data 26-way accuracy and few-shot 10-way accuracy
(10-/20-shot). Our results show the mean $\pm$ standard deviation over 20 seeds (for few-shot, this includes the distinct 10-of-26 letter draws).}
\label{tab:sl}
\begin{tabular}{lccc}
\toprule
Method & Full-data acc & 10-shot & 20-shot \\
\midrule
Point-BERT & 93.25 & 60.29 & 67.35 \\
Point-MAE & 94.16 & 62.85 & 68.96 \\
Point-M2AE & 96.03 & 67.34 & 72.21 \\
UV-Net     & 97.24 & 54.31 & 61.47 \\
BRepNet    & 97.71 & 60.24 & 68.84 \\
BRep-BERT  & \textbf{98.54} & 68.71 & 75.92 \\
\textbf{Ours} & 98.02 $\pm$ 0.06 & \textbf{74.15 $\pm$ 3.37} & \textbf{83.70 $\pm$ 2.73} \\
\bottomrule
\end{tabular}
\end{table}

\subsection{Results}

{\bf Fusion~360 Gallery Segmentation}. We evaluate on Fusion~360 Gallery Segmentation \citep{fusion360} (8-class per-face) under the
fixed public split, in two regimes: full-data fine-tuning and few-shot (8-way
support). 
The closest non-temporal published competitor is BRep-BERT whose F360 representation
is \emph{transductive} (it sees test-set inputs, unlabeled, during pretraining).
Ours is inductive and out-of-domain. A from-scratch control (identical archecture as ours but finetuning from random initialization) isolates the contribution of pretraining in the few-shot rows.

On the full-data setting, our method achieves state-of-the-art macro mIoU among non-temporal methods, despite pretraining entirely out-of-domain. This suggests that our SSL method can generalize and outperforms benchmarks that have seen the F360 data in the pretraining. In the few-shot regime, the from-scratch control lands almost exactly on BRep-BERT's reported number, so pretraining contributes the bulk of the improvement. 

Note that we don't include recent baselines BRepMAE \citep{yao2026brepmae} or HierMae \citep{li2026masked} here. HierMae evaluates F360 on its own
$70/15/15$ split with a different few-shot definition, and BRepMAE reports no
comparable F360 cell. Neither releases its exact split (rather than using the public benchmark splits we adopt), so a controlled
comparison is infeasible. We instead anchor to BRep-BERT, whose protocol we reproduce exactly.

{\bf SolidLetters classification}. SolidLetters, introduced with UV-Net~\citep{jayaraman2021uv}, is a synthetic
B-rep classification benchmark whose 26 classes are the uppercase English letters. We use it in two regimes: (i) the standard
full-data setting, where a model is trained on the full label set and evaluated
26-way on the held-out split, and (ii) a few-shot setting that samples a
10-way episode (10 of the 26 letters) with $K\in\{10,20\}$ labeled solids per
class, measuring how well the frozen representation transfers under tight label budgets.

We compare against the numbers reported in \cite{lou2023brep}, which includes point cloud methods Point-BERT \citep{yu2022point}, Point-MAE \citep{pang2023masked}, and Point-M2AE~\citep{zhang2022point}).
We see that although our method performs slightly worse on full-data than BRep-BERT, it wins considerably on the few-shot setting. We don't compare HierMAE's results here, as its SolidLetters few-shot is $26$-way (all letters) whereas ours is $10$-way, so the numbers are not on the same scale. The data split is not released either, so we therefore compare against BRep-BERT, whose $10$-way protocol we reproduce.

\begin{table}[t]
\centering
\caption{MFInstSeg per-face segmentation accuracy (\%), mean $\pm$ standard deviation across $5$ seeds. Every row is trained and evaluated by us
under an identical protocol.}
\label{tab:mfi}
\begin{tabular}{lcccc}
\toprule
Method  & 0.1\% & 0.5\% & 1\% & Full Dataset \\
\midrule
GAT                 &  $38.67{\pm}0.43$ & $55.50{\pm}1.17$ & $64.29{\pm}1.72$ & $86.74{\pm}0.08$ \\
GraphSAGE           &  $49.02{\pm}0.29$ & $68.09{\pm}0.99$ & $83.83{\pm}0.40$ & $98.26{\pm}0.05$ \\
UV-Net              & $40.06{\pm}5.41$ & $76.54{\pm}0.74$ & $84.97{\pm}0.59$ & $98.77{\pm}0.06$ \\
AAGNet   & $58.63{\pm}1.05$ & $80.20{\pm}2.62$ & $91.40{\pm}0.53$ & $99.30{\pm}0.00$ \\
From-scratch  & $41.26{\pm}1.84$ & $51.04{\pm}0.28$ & $72.10{\pm}1.01$ & $98.54{\pm}0.00$ \\
\textbf{Ours}       
  & {$\mathbf{86.86{\pm}0.95}$}
  & {$\mathbf{93.83{\pm}0.23}$}
  & {$\mathbf{95.95{\pm}0.12}$}
  & {$\mathbf{99.53{\pm}0.01}$} \\
\bottomrule
\end{tabular}
\end{table}

\begin{table}[t]
\centering
\caption{CADSynth label-ratio accuracy (\%). Rows are every method in
HierMAE's CADSynth comparison (their Table~1).}
\label{tab:cadsynth}
\begin{tabular}{lccc}
\toprule
Method & 0.1\% & 0.5\% & 1\% \\
\midrule
AAGNet                   & 58.53 & 87.36 & 95.58 \\
BRepFormer             & 59.34 & 65.60 & 97.49 \\
\textbf{Ours}      & \textbf{93.36 $\pm$ 0.88} & \textbf{97.63 $\pm$ 0.15} & \textbf{98.50 $\pm$ 0.02} \\
\midrule
\multicolumn{4}{l}{\emph{Pretrained on CADSynth (in-distribution)}}\\
BRepMAE            & 93.63 & 96.84 & 98.36 \\
HierMAE           & \underline{93.82} & \underline{98.54} & \underline{98.91} \\
\bottomrule
\end{tabular}
\end{table}

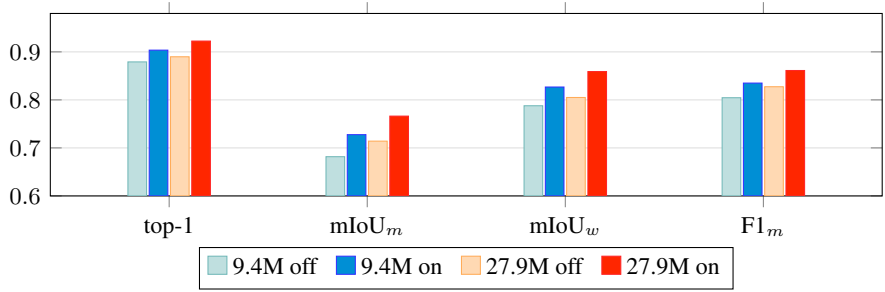
\begin{figure}[t]
\centering
\begin{tikzpicture}
\begin{axis}[
    width=0.9\linewidth, height=4cm,
    ybar=1pt, bar width=7pt,
    ymin=0.6, ymax=0.98,
    ytick={0.6,0.7,0.8,0.9},
    ymajorgrids, grid style={gray!25},
    symbolic x coords={top-1, mIoU$_m$, mIoU$_w$, F1$_m$},
    xtick=data, tick label style={font=\footnotesize},
    enlarge x limits=0.2,
    legend style={font=\footnotesize, at={(0.5,-0.28)}, anchor=north,
              legend columns=4, /tikz/every even column/.append style={column sep=5pt}},
legend image code/.code={\draw[#1] (0cm,-0.08cm) rectangle (0.25cm,0.16cm);},
]
\addplot[fill=teal!25, draw=teal!55] coordinates {
    (top-1,0.8791) (mIoU$_m$,0.6817) (mIoU$_w$,0.7878) (F1$_m$,0.8045)};
\addplot[fill=cyan!75!blue, draw=blue!80] coordinates {
    (top-1,0.9036) (mIoU$_m$,0.7276) (mIoU$_w$,0.8268) (F1$_m$,0.8350)};
\addplot[fill=orange!30, draw=orange!65] coordinates {
    (top-1,0.8898) (mIoU$_m$,0.7141) (mIoU$_w$,0.8050) (F1$_m$,0.8274)};
\addplot[fill=red!70!orange, draw=red!85] coordinates {
    (top-1,0.9225) (mIoU$_m$,0.7662) (mIoU$_w$,0.8590) (F1$_m$,0.8612)};
\legend{9.4M off, 9.4M on, 27.9M off, 27.9M on}
\end{axis}
\end{tikzpicture}
\caption{{\bf MTM On vs. Off}: We plot results on frozen F360 linear probe metrics for both 9.4M and 27.9M models with MTM on vs off. We see that MTM on improves metrics across the board.}
\label{fig:mtm}
\end{figure}

  \begin{table}[t]
  \centering
  \caption{Synthetic-data ablation with a compute-matched control. We show results on Frozen F360 per-face probe. Only corpus and epoch budget differ.}
  \label{tab:synth}
  \begin{tabular}{llcccc}
  \toprule
  Dataset & Steps & Top-1 & Macro mIoU & Weighted mIoU & Macro F1 \\
  \midrule
  ABC-only    & 1.2k  & 0.8399 & 0.6308 & 0.7275 & 0.7626 \\
  ABC-only (compute-matched)  & 16.2k & 0.8471 & 0.6465 & 0.7388 & 0.7737 \\
  ABC$+$synthetic &  16.2k & \textbf{0.8990} & \textbf{0.7219} & \textbf{0.8206} & \textbf{0.8307} \\
  \bottomrule
  \end{tabular}
  \end{table}

{\bf MFInstSeg} \citep{wu2024aagnet} is a B-rep dataset where every face is labeled with one of $25$ machining-feature classes, and the task is per-face segmentation scored by face-level accuracy.
Table~\ref{tab:mfi} compares six models under the official $70/15/15$ split, across $0.1\%/0.5\%/1\%/100\%$ label subsets. Five models are train from scratch: three generic architectures (GAT, GraphSAGE, UV-Net), the task-specific supervised AAGNet (gAAG-GNN), and our own backbone at random init. Our method turns out to be the strongest across the board. We note that our re-implementation of AAGNet reaches $99.30$. This exceeds the $99.15$ per-face accuracy reported in the original paper, which doesn't report the low label budget setting. Our margin is highest as label budget shrinks which suggests that our SSL method is learning useful representation and is label efficient. We note that both BRepMAE and HierMAE are not compared here because they use a different $80/10/10$ split with a different label-subset construction. Neither released the splits.

{\bf CADSynth} \citep{zhang2024brepmfr} is a synthetic B-rep benchmark with dense per-face machining-feature annotations across 25 classes, shipped with an official 80/10/10 train/val/test split. Similar to MFInstSeg, BRepMAE pre-trains and finetune on CADSynth, and HierMae uses CADSynth in its pretraining corpus. Table~\ref{tab:cadsynth} repeats the MFInstSeg protocol on CADSynth. At the extreme $0.1\%$ budget our out-of-distribution model reaches competitive performance, with our best seed ($94.64$) exceeding both, despite never seeing CADSynth before fine-tuning.

The gap to the supervised methods that likewise never pretrain on CADSynth is stark compared to AAGNet and BRepFormer collapse to. As label budget grows, we still clear BRepMAE
but trail  HierMAE. OOD pretraining is
competitive with in-distribution SOTA at the smallest budget, but in-distribution methods pull ahead once labels are less
scarce.

\subsection{Ablations}
\label{sec:ablation}

{\bf Effect of MTM.}
We isolate the effect of MTM using both our $27.9$M architecture and a smaller $9.4$M parameter version of our architecture, training each from
scratch for $30$ epochs with MoCo and region masking enabled in both arms, varying only $\lambda_{\text{topo}}$ ($0.5$ vs.\ $0.0$). On a frozen F360 linear probe (Figure~\ref{fig:mtm}). Since both carry the fully contrastive recipe, this measures MTM's contribution. The frozen probe isolates representation quality. MTM produces significantly better results.

{\bf Synthetic pretraining data}
The MTM ablations in Table~\ref{tab:synth} isolate the effect of adding our procedurally generated dataset. We continue pretraining from a common ABC-only checkpoint (that was pretrained for $45$ epochs from scratch) for $10$ epochs and change only the data source: ABC$+$synthetic vs.\ ABC-only. We hold the architecture, objective, epoch budget, base model, and all hyperparameters fixed. In addition, to adjust for possible effects of more steps because the synthetic dataset is larger, we also run a compute-matched baseline for the ABC-only experiment to match the number of steps as the ABC$+$synth to control for whether the result was merely from more gradient steps. Moreover, the compute-matched control sees a more diverse set of perturbations of the ABC dataset so it's not merely iterating through the same ABC data multiple times. We find that adjusting for the compute is only a marginal gain, while adding our synthetic dataset significantly improves the results.

\section{Conclusion}
We introduced Masked Topology Modeling, an SSL objective that predicts the
convexity and curve type of masked B-rep edges from geometry. 
MTM comes with theoretical guarantees about its ability to recover the edge's convexity and curve type, and forces the encoder to represent geometry that face-only approaches cannot. We combined MTM with MoCo momentum-queue contrastive learning and a face-region masked reconstruction objective, pretrained on ABC and our procedurally generated data, and demonstrated strong empirical performance across four downstream benchmark datasets.

Future work can explore further ways to enrich the MTM objective. For example, performing regression on the angle instead of classifying the angle into $4$ convexity classes; labeling the edge by continuity class (G0: sharp, G1: tangent, G2: curvature-continuous), since this is what fillets and blends actually encode, in order to refine the current single {\it smooth} bucket into finer distinctions \citep{barsky1989geometric}; and predicting whether an edge even exists. Our method is currently an OOD method that matches in-distribution SOTA for the low label regime on CadSynth. Further pretraining experiments including downstream datasets could be a way to improve the results. Moreover, in general, scaling to more data and larger number of parameters to test the capabilities of our approach can be a future direction. Other ways to create synthetic data could also help bridge the gap of too little real-world B-rep available for large-scale training.

\bibliography{references}

\begin{thebibliography}{44}
\providecommand{\natexlab}[1]{#1}
\providecommand{\url}[1]{\texttt{#1}}
\expandafter\ifx\csname urlstyle\endcsname\relax
  \providecommand{\doi}[1]{doi: #1}\else
  \providecommand{\doi}{doi: \begingroup \urlstyle{rm}\Url}\fi

\bibitem[Ansaldi et~al.(1985)Ansaldi, De~Floriani, and Falcidieno]{ansaldi1985geometric}
Silvia Ansaldi, Leila De~Floriani, and Bianca Falcidieno.
\newblock Geometric modeling of solid objects by using a face adjacency graph representation.
\newblock \emph{ACM SIGGRAPH Computer Graphics}, 19\penalty0 (3):\penalty0 131--139, 1985.

\bibitem[Bahri et~al.(2022)Bahri, Jiang, Tay, and Metzler]{bahri2021scarf}
Dara Bahri, Heinrich Jiang, Yi~Tay, and Donald Metzler.
\newblock Scarf: Self-supervised contrastive learning using random feature corruption.
\newblock \emph{ICLR}, 2022.

\bibitem[Barsky \& DeRose(1989)Barsky and DeRose]{barsky1989geometric}
Brian~A Barsky and Tony~D DeRose.
\newblock Geometric continuity of parametric curves: three equivalent characterizations.
\newblock \emph{IEEE Computer Graphics and Applications}, 9\penalty0 (6):\penalty0 60--69, 1989.

\bibitem[Bernardini et~al.(1999)Bernardini, Bajaj, Chen, and Schikore]{bernardini1999automatic}
Fausto Bernardini, Chandrajit~L Bajaj, Jindong Chen, and Daniel~R Schikore.
\newblock Automatic reconstruction of 3d cad models from digital scans.
\newblock \emph{International Journal of Computational Geometry \& Applications}, 9\penalty0 (04n05):\penalty0 327--369, 1999.

\bibitem[Brody et~al.(2021)Brody, Alon, and Yahav]{brody2021attentive}
Shaked Brody, Uri Alon, and Eran Yahav.
\newblock How attentive are graph attention networks?
\newblock \emph{arXiv preprint arXiv:2105.14491}, 2021.

\bibitem[Chen et~al.(2020)Chen, Kornblith, Swersky, Norouzi, and Hinton]{chen2020big}
Ting Chen, Simon Kornblith, Kevin Swersky, Mohammad Norouzi, and Geoffrey~E Hinton.
\newblock Big self-supervised models are strong semi-supervised learners.
\newblock \emph{Advances in neural information processing systems}, 33:\penalty0 22243--22255, 2020.

\bibitem[Cherenkova et~al.(2024)Cherenkova, Dupont, Kacem, Gusev, and Aouada]{cherenkova2024spelsnet}
Kseniya Cherenkova, Elona Dupont, Anis Kacem, Gleb Gusev, and Djamila Aouada.
\newblock Spelsnet: Surface primitive elements segmentation by b-rep graph structure supervision.
\newblock \emph{Advances in Neural Information Processing Systems}, 37:\penalty0 1251--1269, 2024.

\bibitem[Colligan et~al.(2022)Colligan, Robinson, Nolan, Hua, and Cao]{colligan2022hierarchical}
Andrew~R Colligan, Trevor~T Robinson, Declan~C Nolan, Yang Hua, and Weijuan Cao.
\newblock Hierarchical cadnet: Learning from b-reps for machining feature recognition.
\newblock \emph{Computer-Aided Design}, 147:\penalty0 103226, 2022.

\bibitem[Dupont et~al.(2022)Dupont, Cherenkova, Kacem, Ali, Arzhannikov, Gusev, and Aouada]{dupont2022cadops}
Elona Dupont, Kseniya Cherenkova, Anis Kacem, Sk~Aziz Ali, Ilya Arzhannikov, Gleb Gusev, and Djamila Aouada.
\newblock Cadops-net: Jointly learning cad operation types and steps from boundary-representations.
\newblock In \emph{2022 International Conference on 3D Vision (3DV)}, pp.\  114--123. IEEE, 2022.

\bibitem[Faux \& Pratt(1979)Faux and Pratt]{faux1979computational}
Ivor~D Faux and Michael~J Pratt.
\newblock \emph{Computational geometry for design and manufacture}.
\newblock Halsted Press, 1979.

\bibitem[Fujita et~al.(2008)Fujita, Uchiyama, Nakagawa, Fukuoka, Hatanaka, Hara, Lee, Hayashi, Ikedo, Gao, et~al.]{fujita2008computer}
Hiroshi Fujita, Yoshikazu Uchiyama, Toshiaki Nakagawa, Daisuke Fukuoka, Yuji Hatanaka, Takeshi Hara, Gobert~N Lee, Yoshinori Hayashi, Yuji Ikedo, Xin Gao, et~al.
\newblock Computer-aided diagnosis: The emerging of three cad systems induced by japanese health care needs.
\newblock \emph{Computer methods and programs in biomedicine}, 92\penalty0 (3):\penalty0 238--248, 2008.

\bibitem[Groover \& Zimmers(1983)Groover and Zimmers]{groover1983cad}
Mikell Groover and EWJR Zimmers.
\newblock \emph{CAD/CAM: computer-aided design and manufacturing}.
\newblock Pearson Education, 1983.

\bibitem[Guo et~al.(2022)Guo, Liu, Pan, Liu, Tong, and Guo]{guo2022complexgen}
Haoxiang Guo, Shilin Liu, Hao Pan, Yang Liu, Xin Tong, and Baining Guo.
\newblock Complexgen: Cad reconstruction by b-rep chain complex generation.
\newblock \emph{ACM Transactions on Graphics (TOG)}, 41\penalty0 (4):\penalty0 1--18, 2022.

\bibitem[He et~al.(2020)He, Fan, Wu, Xie, and Girshick]{moco}
Kaiming He, Haoqi Fan, Yuxin Wu, Saining Xie, and Ross Girshick.
\newblock Momentum contrast for unsupervised visual representation learning.
\newblock In \emph{CVPR}, 2020.

\bibitem[He et~al.(2022)He, Chen, Xie, Li, Doll{\'a}r, and Girshick]{he2022masked}
Kaiming He, Xinlei Chen, Saining Xie, Yanghao Li, Piotr Doll{\'a}r, and Ross Girshick.
\newblock Masked autoencoders are scalable vision learners.
\newblock In \emph{Proceedings of the IEEE/CVF conference on computer vision and pattern recognition}, pp.\  16000--16009, 2022.

\bibitem[Heesom \& Mahdjoubi(2004)Heesom and Mahdjoubi]{heesom2004trends}
David Heesom and Lamine Mahdjoubi.
\newblock Trends of 4d cad applications for construction planning.
\newblock \emph{Construction management and economics}, 22\penalty0 (2):\penalty0 171--182, 2004.

\bibitem[Jayaraman et~al.(2021)Jayaraman, Sanghi, Lambourne, Willis, Davies, Shayani, and Morris]{jayaraman2021uv}
Pradeep~Kumar Jayaraman, Aditya Sanghi, Joseph~G Lambourne, Karl~DD Willis, Thomas Davies, Hooman Shayani, and Nigel Morris.
\newblock Uv-net: Learning from boundary representations.
\newblock In \emph{Proceedings of the IEEE/CVF conference on computer vision and pattern recognition}, pp.\  11703--11712, 2021.

\bibitem[Jayaraman et~al.(2022)Jayaraman, Lambourne, Desai, Willis, Sanghi, and Morris]{jayaraman2022solidgen}
Pradeep~Kumar Jayaraman, Joseph~G Lambourne, Nishkrit Desai, Karl~DD Willis, Aditya Sanghi, and Nigel~JW Morris.
\newblock Solidgen: An autoregressive model for direct b-rep synthesis.
\newblock \emph{arXiv preprint arXiv:2203.13944}, 2022.

\bibitem[Joshi \& Chang(1988)Joshi and Chang]{joshi1988graph}
Sanjay Joshi and Tien-Chien Chang.
\newblock Graph-based heuristics for recognition of machined features from a 3d solid model.
\newblock \emph{Computer-aided design}, 20\penalty0 (2):\penalty0 58--66, 1988.

\bibitem[Jung et~al.(2024)Jung, Kim, and Kim]{jung2024contrastcad}
Minseop Jung, Minseong Kim, and Jibum Kim.
\newblock Contrastcad: Contrastive learning-based representation learning for computer-aided design models.
\newblock \emph{arXiv preprint arXiv:2404.01645}, 2024.

\bibitem[Kantaros(2024)]{kantaros2024intellectual}
Antreas Kantaros.
\newblock Intellectual property challenges in the age of 3d printing: Navigating the digital copycat dilemma.
\newblock \emph{Applied Sciences}, 14\penalty0 (23):\penalty0 11448, 2024.

\bibitem[Karahan et~al.(2025)Karahan, Velio{\u{g}}lu, and Arslan]{karahan2025unveiling}
{\c{C}}etin Karahan, Hakan Velio{\u{g}}lu, and Yenal Arslan.
\newblock Unveiling the shadows: Anticipating future cyber risks and their impacts on businesses.
\newblock In \emph{Futurisks: Risk management in the digital age}, pp.\  17--55. Springer, 2025.

\bibitem[Koch et~al.(2019)Koch, Matveev, Jiang, Williams, Artemov, Burnaev, Alexa, Zorin, and Panozzo]{abc}
Sebastian Koch, Albert Matveev, Zhongshi Jiang, Francis Williams, Alexey Artemov, Evgeny Burnaev, Marc Alexa, Denis Zorin, and Daniele Panozzo.
\newblock {ABC}: A big {CAD} model dataset for geometric deep learning.
\newblock \emph{CVPR}, 2019.

\bibitem[Lambourne et~al.(2021)Lambourne, Willis, Jayaraman, Sanghi, Meltzer, and Shayani]{lambourne2021brepnet}
Joseph~G Lambourne, Karl~DD Willis, Pradeep~Kumar Jayaraman, Aditya Sanghi, Peter Meltzer, and Hooman Shayani.
\newblock Brepnet: A topological message passing system for solid models.
\newblock In \emph{Proceedings of the IEEE/CVF conference on computer vision and pattern recognition}, pp.\  12773--12782, 2021.

\bibitem[Lecallard et~al.(2021)Lecallard, Robinson, and Marques]{lecallard2021mesh}
Benoit Lecallard, Trevor~T Robinson, and Simao~P Marques.
\newblock Mesh and geometry manipulations for optimization and inverse design.
\newblock In \emph{AIAA Scitech 2021 Forum}, pp.\  1901, 2021.

\bibitem[Li et~al.(2025)Li, Yu, Li, Liu, and Yang]{li2025enhancing}
Chengnan Li, Tong Yu, Wangyan Li, Ying Liu, and Hongjun Yang.
\newblock Enhancing cad data integrity and security in supply chain networks using blockchain.
\newblock \emph{International Journal of Information Systems and Supply Chain Management (IJISSCM)}, 18\penalty0 (1):\penalty0 1--22, 2025.

\bibitem[Li et~al.(2026)Li, Wu, Wu, and Fu]{li2026masked}
Yifei Li, Kang Wu, Wenming Wu, and Xiao-Ming Fu.
\newblock Masked brep autoencoder via hierarchical graph transformer.
\newblock \emph{arXiv preprint arXiv:2603.14927}, 2026.

\bibitem[Liu(2021)]{liu2021fast}
Fei Liu.
\newblock Fast industrial product design method and its application based on 3d cad system.
\newblock \emph{Computer-Aided Design \& Applications}, 18, 2021.

\bibitem[Lou et~al.(2023)Lou, Li, Chen, and Zhou]{lou2023brep}
Yunzhong Lou, Xueyang Li, Haotian Chen, and Xiangdong Zhou.
\newblock Brep-bert: Pre-training boundary representation bert with sub-graph node contrastive learning.
\newblock In \emph{Proceedings of the 32nd ACM International Conference on Information and Knowledge Management}, pp.\  1657--1666, 2023.

\bibitem[Oord et~al.(2018)Oord, Li, and Vinyals]{oord2018representation}
Aaron van~den Oord, Yazhe Li, and Oriol Vinyals.
\newblock Representation learning with contrastive predictive coding.
\newblock \emph{arXiv preprint arXiv:1807.03748}, 2018.

\bibitem[Pang et~al.(2023)Pang, Tay, Yuan, and Chen]{pang2023masked}
Yatian Pang, Eng Hock~Francis Tay, Li~Yuan, and Zhenghua Chen.
\newblock Masked autoencoders for 3d point cloud self-supervised learning.
\newblock \emph{World Scientific Annual Review of Artificial Intelligence}, 1:\penalty0 2440001, 2023.

\bibitem[Quan et~al.(2024)Quan, Zhao, Yi, and Chen]{quan2024self}
Yuhan Quan, Huan Zhao, Jinfeng Yi, and Yuqiang Chen.
\newblock Self-supervised graph neural network for mechanical cad retrieval.
\newblock \emph{arXiv preprint arXiv:2406.08863}, 2024.

\bibitem[Rukhovich et~al.(2025)Rukhovich, Dupont, Mallis, Cherenkova, Kacem, and Aouada]{cadrecode}
Danila Rukhovich, Elona Dupont, Dimitrios Mallis, Kseniya Cherenkova, Anis Kacem, and Djamila Aouada.
\newblock Cad-recode: Reverse engineering cad code from point clouds.
\newblock In \emph{Proceedings of the IEEE/CVF International Conference on Computer Vision}, pp.\  9801--9811, 2025.

\bibitem[Stjepandi{\'c} et~al.(2015)Stjepandi{\'c}, Liese, and Trappey]{stjepandic2015intellectual}
Josip Stjepandi{\'c}, Harald Liese, and Amy~JC Trappey.
\newblock Intellectual property protection.
\newblock In \emph{Concurrent engineering in the 21st century: Foundations, developments and challenges}, pp.\  521--551. Springer, 2015.

\bibitem[Szalapaj(2013)]{szalapaj2013cad}
Peter Szalapaj.
\newblock \emph{CAD principles for architectural design}.
\newblock Routledge, 2013.

\bibitem[Vaswani et~al.(2017)Vaswani, Shazeer, Parmar, Uszkoreit, Jones, Gomez, Kaiser, and Polosukhin]{vaswani2017attention}
Ashish Vaswani, Noam Shazeer, Niki Parmar, Jakob Uszkoreit, Llion Jones, Aidan~N Gomez, {\L}ukasz Kaiser, and Illia Polosukhin.
\newblock Attention is all you need.
\newblock \emph{Advances in neural information processing systems}, 30, 2017.

\bibitem[Willis et~al.(2021)Willis, Pu, Luo, Chu, Du, Lambourne, Solar-Lezama, and Matusik]{fusion360}
Karl~D.D. Willis, Yewen Pu, Jieliang Luo, Hang Chu, Tao Du, Joseph~G. Lambourne, Armando Solar-Lezama, and Wojciech Matusik.
\newblock Fusion 360 gallery: A dataset and environment for programmatic {CAD} construction.
\newblock In \emph{SIGGRAPH}, 2021.

\bibitem[Willis et~al.(2022)Willis, Jayaraman, Chu, Tian, Li, Grandi, Sanghi, Tran, Lambourne, Solar-Lezama, et~al.]{willis2022joinable}
Karl~DD Willis, Pradeep~Kumar Jayaraman, Hang Chu, Yunsheng Tian, Yifei Li, Daniele Grandi, Aditya Sanghi, Linh Tran, Joseph~G Lambourne, Armando Solar-Lezama, et~al.
\newblock Joinable: Learning bottom-up assembly of parametric cad joints.
\newblock In \emph{Proceedings of the IEEE/CVF conference on computer vision and pattern recognition}, pp.\  15849--15860, 2022.

\bibitem[Wu et~al.(2024)Wu, Lei, Peng, and Gao]{wu2024aagnet}
Hongjin Wu, Ruoshan Lei, Yibing Peng, and Liang Gao.
\newblock Aagnet: A graph neural network towards multi-task machining feature recognition.
\newblock \emph{Robotics and Computer-Integrated Manufacturing}, 86:\penalty0 102661, 2024.

\bibitem[Xu et~al.(2025)Xu, Jayaraman, Lambourne, Liu, Malpure, and Meltzer]{xu2025autobrep}
Xiang Xu, Pradeep Jayaraman, Joseph Lambourne, Yilin Liu, Durvesh Malpure, and Pete Meltzer.
\newblock Autobrep: Autoregressive b-rep generation with unified topology and geometry.
\newblock In \emph{Proceedings of the SIGGRAPH Asia 2025 Conference Papers}, pp.\  1--12, 2025.

\bibitem[Yao et~al.(2026)Yao, Wu, Zheng, Xing, and Fu]{yao2026brepmae}
Can Yao, Kang Wu, Zuheng Zheng, Siyuan Xing, and Xiao-Ming Fu.
\newblock Brepmae: Self-supervised masked brep autoencoders for machining feature recognition.
\newblock \emph{arXiv preprint arXiv:2602.22701}, 2026.

\bibitem[Yu et~al.(2022)Yu, Tang, Rao, Huang, Zhou, and Lu]{yu2022point}
Xumin Yu, Lulu Tang, Yongming Rao, Tiejun Huang, Jie Zhou, and Jiwen Lu.
\newblock Point-bert: Pre-training 3d point cloud transformers with masked point modeling.
\newblock In \emph{Proceedings of the IEEE/CVF conference on computer vision and pattern recognition}, pp.\  19313--19322, 2022.

\bibitem[Zhang et~al.(2022)Zhang, Guo, Gao, Fang, Zhao, Wang, Qiao, and Li]{zhang2022point}
Renrui Zhang, Ziyu Guo, Peng Gao, Rongyao Fang, Bin Zhao, Dong Wang, Yu~Qiao, and Hongsheng Li.
\newblock Point-m2ae: multi-scale masked autoencoders for hierarchical point cloud pre-training.
\newblock \emph{Advances in neural information processing systems}, 35:\penalty0 27061--27074, 2022.

\bibitem[Zhang et~al.(2024)Zhang, Guan, Jiang, Wang, and Tan]{zhang2024brepmfr}
Shuming Zhang, Zhidong Guan, Hao Jiang, Xiaodong Wang, and Pingan Tan.
\newblock Brepmfr: Enhancing machining feature recognition in b-rep models through deep learning and domain adaptation.
\newblock \emph{Computer Aided Geometric Design}, 111:\penalty0 102318, 2024.

\end{thebibliography}
\bibliographystyle{paper}

\appendix

\section{Architecture Details}
\label{app:backbone}

A neural network cannot consume the graph directly, so we must turn each face into a fixed-size vector. UV-grid sampling does this without
meshing. Every parametric surface comes with a built-in two-dimensional coordinate system
$(u,v)$ and a map $S(u,v)\in\mathbb{R}^3$ that sends a point of a flat parameter
rectangle to a point on the surface in 3D. For a cylinder, for example, $u$ is
the angle around the axis and $v$ is the height. Even though the face is curved
in space, its parameter domain is flat, so we can lay a regular grid over that
domain.

We place a fixed $10\times 10$ grid of $(u,v)$ values over the face's parameter domain and evaluate the surface at each grid point. At each of the $100$ samples, we store $7$ channels: position ($3$-dim): the 3D point $S(u,v)$; normal ($3$-dim): the unit outward normal $n(u,v)$; trimming mask ($1$-dim): whether the sample lies inside the trimmed face or in the cut-away region, since the parameter rectangle is larger than the actual face. The result is a $10\times 10\times 7$ tensor. A surface CNN processes it into a single per-face embedding $h_i$, which is the \emph{node feature} in our face-adjacency graph.

A B-rep edge is a parametric curve with a single parameter $u$. We sample it on a length-$10$ grid and store $6$ channels per sample: $3$ for position and $3$ for the curve tangent, giving a $10\times 6$ tensor, which a 1D curve CNN turns into an edge embedding.

We now explain our architecture choices.
Our encoder ($27.9$M params, hidden width $H{=}384$) maps a
B-rep face-graph to (i)~a shape embedding $z\in\mathbb{R}^{512}$ and (ii)~a
per-face feature matrix $h\in\mathbb{R}^{N_F\times 384}$. The full data flow, with
exact layer specifications (input tensors in Appendix~\ref{app:preproc}):

\begin{enumerate}
\item \textbf{Face encoder (per-face CNN).} Each face's $10\times10\times7$
UV-grid is permuted to $[7,10,10]$ and passed through three $3\times3$ conv blocks
($\mathrm{Conv2d(pad\,1)}\to\mathrm{GroupNorm}\to\mathrm{GELU}$) with channel
widths $128\to H\to H$ (GroupNorm groups $8,16,16$), then
$\mathrm{AdaptiveAvgPool2d}(1)$, flatten, and $\mathrm{Linear}(H,H)$, giving a
face token in $\mathbb{R}^H$.
\item \textbf{Edge encoder (per-edge CNN).} Each edge's $10\times6$ UV-grid is
permuted to $[6,10]$ and passed through two $\mathrm{Conv1d}(3,\text{pad}\,1)\to
\mathrm{GroupNorm}\to\mathrm{GELU}$ blocks ($128\to H$),
$\mathrm{AdaptiveAvgPool1d}(1)$, flatten. To this we concatenate two $16$-d
learned embeddings of the edge's convexity class ($4$) and curve-type class ($5$),
then $\mathrm{Linear}(H{+}32,H)$. Each directed graph edge reads its parent edge
token by index.
\item \textbf{Graph message passing.} $N_{\text{gat}}{=}4$ GATv2 layers
\citep{brody2021attentive} ($\mathrm{GATv2Conv}(H\to H,\text{heads}{=}6,\text{concat}{=}
\text{False},\text{edge\_dim}{=}H)$), each using the edge token as an additive
attention feature, as a pre-residual with LayerNorm:
$x\leftarrow\mathrm{LayerNorm}(x+\mathrm{GATv2}(x,\text{edge\_index},
\text{edge\_attr}))$.
\item \textbf{Global self-attention.} Per-graph face tokens are padded to a dense
batch (with a key-padding mask) and passed through $N_{\text{tf}}{=}8$
TransformerEncoder layers ($d_{\text{model}}{=}384$, $\text{nhead}{=}6$,
$\text{dim\_feedforward}{=}1536$, GELU, $\text{norm\_first}{=}\text{True}$). The
per-face features $h$ are read out here (post-Transformer, unpadded) --- these
feed the MTM and region heads.
\item \textbf{Attention pooling.} A single learned query attends over the face
tokens via $\mathrm{MultiheadAttention}$ ($6$ heads), followed by LayerNorm and
$\mathrm{Linear}(384,512)$ to produce the shape embedding $z$.
\end{enumerate}

The $8$-layer Transformer depth matches the BRep-BERT scale (${\sim}22$M) and is
required to reach its F360 top-1; shallower variants plateau lower. The $9.4$M
ablation testbed is the same stack with $H{=}256$,
$N_{\text{tf}}{=}2$, $\text{nhead}{=}8$, $\text{ff}{=}1024$, MoCo off. Both scales
are instantiated by the same code; the only differences are the
width/depth/head hyperparameters in Appendix~\ref{app:hparams}.

\section{Contrastive Learning}
\label{app:contrastive}

We use momentum-queue contrastive learning~\citep{moco} at the shape-embedding
level. Each shape is encoded into two augmented views; view~A is encoded
by the online encoder $f_q$ (query $q{=}z_a$), view~B by a momentum encoder $f_k$
whose weights are an EMA of the online weights,
$\theta_k\leftarrow m\,\theta_k+(1{-}m)\theta_q$ with $m{=}0.999$. The momentum
encoder's outputs (key $k^+{=}z_b$) are enqueued into a FIFO queue of
$K{=}16{,}384$ past keys. The query is trained with InfoNCE,
\begin{equation}
L_{\text{NCE}} = -\log\frac{\exp(q\cdot k^+/\tau)}
{\exp(q\cdot k^+/\tau)+\sum_{k^-\in\text{queue}}\exp(q\cdot k^-/\tau)},
\end{equation}
temperature $\tau{=}0.07$. Each anchor contrasts its positive against in-batch
negatives (${\sim}63$ at batch $64$) \emph{plus} the full $16{,}384$-entry queue
--- roughly a $250\times$ increase over batch-only contrastive at the memory cost
of a single feature bank. The momentum encoder decouples key staleness from batch
size, making the large queue usable on a single node. BRep-BERT uses the same $K$
for a like-for-like contrastive scale.

\section{Method Details}

\subsection{MTM Architecture}
\label{app:mtm-details}
On view B, we sample a masked edge set $M\subset E$ with
$|M|{=}\max(1,\mathrm{round}(p_{\text{mask}}|E|))$, $p_{\text{mask}}{=}0.7$, and
delete it from the graph's edge index before the encoder forward. The encoder
$\phi$ runs GATv2 + Transformer layers over the masked graph, producing
per-face features
$h_i=\phi_i(\text{face grids};(F,E\setminus M))\in\mathbb{R}^{384}$. For each
masked edge $e{=}(i,j)\in M$ we form the ordered endpoint feature
$[h_i;h_j]\in\mathbb{R}^{768}$ and pass it through two independent heads, each
$\mathrm{Linear}(768,128)\!\to\!\mathrm{GELU}\!\to\!\mathrm{Dropout}(0.1)\!\to\!
\mathrm{Linear}(128,\cdot)$, with $4$ and $5$ output classes respectively. Both
heads are discarded after pretraining (${\approx}0.1$M params total). Masking is
applied to view~B only, because view~B feeds the auxiliary heads while view~A
carries the heavy contrastive augmentation.

\subsection{Augmentation Details}
\label{app:synth}
\textbf{Offline geometric perturbations.} For each source solid we precompute
$n_{\text{perts}}{=}5$ perturbed variants by applying one randomly chosen B-rep
operator (via the OpenCASCADE kernel), each validated for a watertight solid. The
operator is sampled from $14$ choices with fixed weights (full table below);
feature sizes are relative to the shape's bounding-box diagonal. A shape enters
pretraining only if $\geq\!\text{min\_perts}{=}2$ variants preprocessed
successfully; a training pair draws two \emph{different} stored variants of the
same source.

\textbf{Online graph augmentations.} Per view we apply: \emph{face drop}
(fraction $p_{\text{face}}{=}0.3$ of faces and incident edges, only if
$N_F{>}4$, capped at $12$); \emph{edge drop} ($p_{\text{edge}}{=}0.3$ of graph edges,
capped at $30$); \emph{connected-subgraph sampling} (with probability $0.5$, BFS from a
random seed keeping the first $\lceil 0.7\,N_F\rceil$ faces); \emph{coordinate
jitter} (additive Gaussian, std $0.005$ of normalized scale, on xyz); and
\emph{asymmetric masking} --- view~A receives the full heavy augmentation, view~B
only light augmentation ($5\%$ face drop, no edge drop or subgraph sampling).

\subsection{Face-Region Masked Reconstruction Details}
\label{app:region}
Applied to view~B, in four steps:
\begin{enumerate}
\item \textbf{Select region.} For each shape with $N_F{>}4$, BFS from a random
seed face along the face-adjacency graph, accumulating until
$\lceil 0.4\,N_F\rceil$ faces are covered (region ratio $r{=}0.4$) --- a single
\emph{connected} blob.
\item \textbf{Mask.} Replace the $10\times10\times7$ UV-grid of every selected
face with a single learnable region mask token $\in\mathbb{R}^7$ (init
$\mathcal{N}(0,0.02^2)$, broadcast over the grid) before the encoder forward.
\item \textbf{Reconstruct.} From each masked face's feature $h_i$, two heads
predict: geometry, via $\mathrm{Linear}(H,H)\to\mathrm{GELU}\to
\mathrm{Dropout}(0.1)\to\mathrm{Linear}(H,600)$ reshaped to $10\times10\times6$
(xyz+normal), trained with MSE against the original grid; and coarse geometric
face type (planar / developable / doubly-curved / freeform;
Appendix~\ref{app:preproc}) via $\mathrm{Linear}(H,128)\to\mathrm{GELU}\to
\mathrm{Dropout}(0.1)\to\mathrm{Linear}(128,4)$, trained with cross-entropy.
\item \textbf{Loss.} $L_{\text{region}}=\mathrm{MSE}(\text{geom})+
\mathrm{CE}(\text{type})$, weight $\lambda_{\text{region}}{=}1.0$.
\end{enumerate}

\section{Proofs for Section~\ref{sec:theory}}
\label{app:theory}

Throughout, the MTM head is scored by cross-entropy, and we use the standard fact
that for a target $Y$ and inputs $X$ the smallest expected cross-entropy over
measurable predictors equals the conditional entropy $H(Y\mid X)$, attained by the
true posterior $p(Y\mid X)$; in particular this minimum is $0$ iff $Y$ is
$\sigma(X)$-measurable (i.e.\ determined by $X$). The same conclusions hold verbatim
for $0/1$ loss with $H$ replaced by the Bayes error. Outward unit normals point away
from the material (interior) side, so the samples encode which side is solid.

\subsection*{Proof of Theorem~\ref{thm:ordering} (Identifiability)}

\begin{lemma}[Convexity is a local function of the normals]
\label{lem:conv}
Under \textup{(A1)}--\textup{(A2)}, the convexity $c(e)$ is a deterministic function
of the outward normal field of the two faces in an arbitrarily small neighborhood of
$\sigma_e$.
\end{lemma}

\begin{proof}
By definition, $c(e)$ is the class of the material-side dihedral angle $\theta(p)$
between the two faces along the shared curve: $\theta<\pi$ convex, $\theta>\pi$
concave, $\theta=\pi$ smooth (tangent), and the degenerate $\theta\in\{0,2\pi\}$
knife. At a point $p\in\sigma_e$, let $\hat t$ be a unit tangent to $\sigma_e$ and
$n_i(p),n_j(p)$ the two outward unit normals, which exist by (A1). The conormals
$m_k = \hat t\times n_k$ lie in each face's tangent plane and point across the edge;
$\theta(p)$ is the angle from $m_i$ to $m_j$ measured through the exterior, an
explicit continuous function of $\big(n_i(p),n_j(p),\hat t\big)$. By (A2) the edge
carries a single such curve, so no other adjacency interferes and $c(e)$ is the
class of $\theta$, a function of the normals on any neighborhood of $\sigma_e$.
\end{proof}

\paragraph{Convexity: infimum loss $0$.}
Consider the finite-resolution encoder that sets $h_v$ to a lossless encoding of
face $v$'s samples (permitted under (A3), which only upper-bounds resolution). By
(A3) those samples resolve the outward normals in a neighborhood of $\sigma_e$, so
the pair $(h_i,h_j)$ contains $n_i,n_j$ near $\sigma_e$; by Lemma~\ref{lem:conv}
there is a measurable $\Psi$ with $c(e)=\Psi(h_i,h_j)$. Then $H(c(e)\mid h_i,h_j)=0$
and the head equal to $\Psi$ attains cross-entropy $0$. Hence the smallest
achievable convexity loss is $0$.

\paragraph{Type: infimum loss $\ge \mathbb{E}[H(t(e)\mid\mathcal S_e)]>0$.}
Fix \emph{any} finite-resolution encoder. By (A3) each $h_v$ is a function of face
$v$'s finite sample set, so $(h_i,h_j)=g(\mathcal S_e)$ for some map $g$, where
$\mathcal S_e$ collects the samples of the two faces of $e$. The minimal type loss of
this encoder is $\mathbb{E}\!\left[H(t(e)\mid h_i,h_j)\right]$, and by the
data-processing inequality (conditioning on a function of $\mathcal S_e$ cannot
reduce entropy below conditioning on $\mathcal S_e$),
\[
\mathbb{E}\!\left[H(t(e)\mid h_i,h_j)\right]
\;\ge\;
\mathbb{E}\!\left[H(t(e)\mid \mathcal S_e)\right].
\]
By (A4) there is a positive-probability event on which $t(e)$ is not determined by
$\mathcal S_e$, i.e.\ $H(t(e)\mid\mathcal S_e)>0$ there; hence the right-hand side is
strictly positive. As the bound holds for every finite-resolution encoder, the
smallest achievable type loss is at least $\mathbb{E}[H(t(e)\mid\mathcal S_e)]>0$.

\paragraph{Ordering.}
The convexity floor is $0$ and the type floor is at least
$\mathbb{E}[H(t(e)\mid\mathcal S_e)]>0$, so the convexity head is trainable to a
strictly smaller loss than the type head. The type Bayes posterior puts its residual
mass exactly on the event of (A4)---curves a simpler analytic form matches at the
sampled resolution (near-circular ellipses, large-radius arcs, splines that locally
mimic conics). \qed

\subsection*{Proof of Theorem~\ref{thm:separation} (Separability)}

\paragraph{The witness (bump vs.\ dimple), explicit coordinates.}
Let the base solid be the slab $B=[-2,2]\times[-2,2]\times[-2,0]$, with top face at
$z=0$, and let $Q=[-1,1]\times[-1,1]\times\{0\}$ be a square inscribed in it. Fix any
height $0<h<2$ and apex $a=(0,0,h)$, $a'=(0,0,-h)$.
\begin{itemize}\itemsep2pt
  \item $S$ (\emph{bump}) $= B$ union the right pyramid with base $Q$ and apex $a$.
  \item $S'$ (\emph{dimple}) $= B$ minus the open right pyramid with base $Q$ and
  apex $a'$.
\end{itemize}
Both are valid closed solids: $h<2$ keeps the dimple apex $a'$ strictly above the
slab bottom $z=-2$, so $S'$ does not self-intersect.

\emph{Face bijection $\phi$.} Bottom, the four slab sides, and the top face (the
square frame $[-2,2]^2\setminus Q$ at $z=0$) occur identically in $S$ and $S'$; map
each to itself. The remaining faces are four triangles. In $S$ the east triangle is
$T_E=\mathrm{conv}\{(1,1,0),(1,-1,0),a\}$ (and analogously $T_S,T_W,T_N$); in $S'$ it
is $T_E'=\mathrm{conv}\{(1,1,0),(1,-1,0),a'\}$. Set $\phi(T_E)=T_E'$, etc.

\emph{$\phi$ preserves intrinsic geometry.} $T_E'$ is the reflection of $T_E$ through
$z=0$, hence congruent; each triangle is isosceles (base edge length $2$, two slant
edges $\sqrt{2+h^2}$), so it is congruent to its mirror image by a proper rotation
(about the perpendicular bisector of its base), and $\gamma(f)=\gamma(\phi(f))$ even
if $\gamma$ is taken up to orientation-preserving isometry. All non-triangular faces
are literally identical.

\emph{$\phi$ is a graph isomorphism.} In both solids the four triangles share slant
(``hip/valley'') edges pairwise at the apex and share their base edges with the top
frame; $B$'s faces meet in the same pattern. The incidence pattern is identical, so
$\phi\colon G(S)\to G(S')$ is an isomorphism.

\emph{A convexity flip.} Take $e$ = the slant edge shared by $T_E,T_S$, from corner
$(1,-1,0)$ to $a$. In $S$ this is a hip of a convex bump, so its material-side
dihedral is $<\pi$: $c(e)=\text{convex}$. Its image $\phi(e)$ runs from $(1,-1,0)$ to
$a'$ and is a valley of the indentation, material-side dihedral $>\pi$:
$c(\phi(e))=\text{concave}$. Thus $c(e)\neq c(\phi(e))$ while $(\gamma,G)$ agree,
establishing the hypotheses of the theorem.

\paragraph{(i) Face-level objectives admit a merging optimum.}
Let $J$ be any objective whose value depends on the representation only through
$(\gamma,G)$-measurable quantities. Suppose $R^\ast$ is a global optimum. Because
$\phi$ matches intrinsic geometry and adjacency, corresponding faces $(S,f)$ and
$(S',\phi(f))$ present identical $(\gamma,G)$ inputs. Define the symmetrized
representation $R^{\ast\ast}$ by $R^{\ast\ast}(S)_f=R^{\ast\ast}(S')_{\phi(f)}
:=R^{\ast}(S)_f$. Since $J$ reads only $(\gamma,G)$, which are unchanged by this
reassignment, $J(R^{\ast\ast})=J(R^{\ast})$, so $R^{\ast\ast}$ is also a global
optimum; it satisfies $h_f(S)=h_{\phi(f)}(S')$ for all $f$. On this optimum $S$ and
$S'$ carry identical embedding multisets and are indistinguishable to any readout.

\paragraph{(ii) A zero-MTM-loss representation must separate $S,S'$.}
First, no $(\gamma,G)$-measurable map $T$ is a sufficient statistic for $c$: such a
$T$ takes equal values on $e$ and $\phi(e)$ (identical $(\gamma,G)$), yet
$c(e)\neq c(\phi(e))$, so $c$ is not $\sigma(T)$-measurable. Now suppose a
representation attains zero MTM loss. Mask $e$ in $S$ and the corresponding
$\phi(e)$ in $S'$ (the masked graphs correspond under $\phi$). Zero loss forces the
head to output $c(e)$ and $c(\phi(e))$ exactly from the endpoint embeddings. If the
representation were $(\gamma,G)$-measurable, the endpoint embeddings for $e$ and
$\phi(e)$ would coincide (same intrinsic geometry, same position in the abstract
graph), forcing one head output for both---contradicting $c(e)\neq c(\phi(e))$.
Hence the representation is \emph{not} $(\gamma,G)$-measurable: it assigns distinct
embeddings to at least one corresponding endpoint pair, so $h_f(S)\neq h_{\phi(f)}(S')$
for some endpoint $f$ of $e$, and $S,S'$ are distinguishable.

\paragraph{Conclusion.}
By (i) a face-level objective admits an optimum merging $S$ and $S'$, while by (ii)
every zero-MTM-loss representation separates them. Thus on this class of solids
MTM-optimal representations are strictly more discriminative than any face-level
one. \qed

\section{Pretraining hyperparameters}
\label{app:hparams}
\begin{table}[h]
\centering
\begin{tabular}{lcc}
\toprule
Parameter & Headline ($27.9$M) & Ablation testbed ($9.4$M) \\
\midrule
Backbone hidden $H$ & 384 & 256 \\
GATv2 layers / heads & 4 / 6 & 4 / 8 \\
Transformer layers / heads / FFN & 8 / 6 / 1536 & 2 / 8 / 1024 \\
Output (shape-embed) dim & 512 & 512 \\
Parameter count & 27.9M & 9.43M \\
MoCo queue $K$ / momentum & 16{,}384 / 0.999 & off (0) \\
InfoNCE temperature $\tau$ & 0.07 & 0.07 \\
topo\_mask\_prob / topo\_aux\_weight & 0.7 / 0.5 & 0.7 / 0.3 \\
region\_mask\_weight / region ratio & 1.0 / 0.4 & 1.0 / 0.4 \\
Face drop $p_{\text{face}}$ (abs cap) & 0.3 (12) & 0.3 (12) \\
Edge drop $p_{\text{edge}}$ (abs cap) & 0.3 (30) & 0.3 (30) \\
Subgraph trigger / keep ratio & 0.5 / 0.7 & 0.5 / 0.7 \\
Coordinate jitter std (frac) & 0.005 & 0.005 \\
Optimizer & AdamW $\beta{=}(0.9,0.95)$ & AdamW $\beta{=}(0.9,0.95)$ \\
LR / weight decay & 1e-4 / 0.01 & 1e-4 / 0.01 \\
Schedule & 3-ep warmup $\to$ const & same \\
Epochs (from scratch) & 30 & 30 \\
Per-GPU batch $\times$ accum $\times$ ranks & $64\times16\times2$ & $64\times32\times1$ \\
Effective batch & 2{,}048 & 2{,}048 \\
Precision & bf16 autocast & bf16 autocast \\
Max faces / shape & 400 & 400 \\
Perturbations ($n_{\text{perts}}$ / $\text{min}$) & 5 / 2 & 5 / 2 \\
\bottomrule
\end{tabular}
\end{table}
\textbf{Combined loss.} $L = L_{\text{NCE}} + \lambda_{\text{topo}}L_{\text{topo}}
+ \lambda_{\text{region}}L_{\text{region}}$ with $\lambda_{\text{topo}}{=}0.5$,
$\lambda_{\text{region}}{=}1.0$. $L_{\text{NCE}}$ is MoCo InfoNCE (headline) or
symmetric in-batch NT-Xent ($9.4$M, MoCo off). Auxiliary heads are discarded after
pretraining; only the backbone transfers downstream.

\section{Preprocessing and input representation}
\label{app:preproc}
Every solid (a STEP file) is converted to a face-graph with UV-grid features using
the OpenCASCADE (OCC) kernel. All geometry below is deterministic given the solid;
no learned or stochastic component enters preprocessing.

\textbf{Faces $\to$ nodes.} For each face we build a $10\times10\times7$ grid by
sampling the parametric surface on a uniform $10\times10$ lattice over
$[u_{\min},u_{\max}]\times[v_{\min},v_{\max}]$ (degenerate ranges clamped to unit
length). The $7$ channels are: xyz position ($3$), unit surface normal ($3$,
flipped when the face orientation is \texttt{REVERSED} so it points outward), and a
trim mask ($1$) that is $1$ when the sample lies inside/on the face's trimmed
region (\texttt{BRepClass\_FaceClassifier}) and $0$ otherwise.

\textbf{Edges $\to$ typed graph edges.} For each B-rep edge adjacent to $\geq 2$
faces we build a $10\times6$ grid by sampling the edge curve on $10$ uniform
parameter values: xyz position ($3$) and unit tangent ($3$). The edge is
instantiated as directed graph edges between every ordered pair of its adjacent
faces. Each carries two categorical labels:
\begin{itemize}
\item \textbf{Convexity} $\in\{\text{concave}{=}0,\text{convex}{=}1,\text{smooth}
{=}2,\text{knife}{=}3\}$, computed at the edge midpoint from adjacent-face normals
$n_a,n_b$ and tangent $t$: $n_a\cdot n_b>0.999\to$ smooth, $<-0.999\to$ knife, else
$\mathrm{sign}((n_a\times n_b)\cdot t)$ gives convex ($+$) vs.\ concave ($-$).
\item \textbf{Curve type} $\in\{\text{Line}{=}0,\text{Circle}{=}1,\text{B-spline}
{=}2,\text{Ellipse}{=}3,\text{Other}{=}4\}$ from the OCC curve class
(\texttt{GeomAbs\_*}); B\'ezier/hyperbola/parabola/unknown fold into ``Other''.
\end{itemize}
These are exactly the labels MTM predicts; at input time they also
condition the edge CNN via $16$-d learned embeddings.

\textbf{Normalization.} xyz channels of both face and edge grids are centered on
the shape's bounding-box center and divided by its bounding-box diagonal, giving a
unit-diagonal shape. Grids are stored in fp16.

\textbf{Coarse face type (region target).} For the region objective we derive a
$4$-way face type from the PCA spread of the face's grid normals: \textbf{planar}
(total normal spread $<0.005$), \textbf{developable} (smallest eigenvalue $<0.01$,
largest $>0.03$ --- cylinders/cones), \textbf{doubly-curved} (all eigenvalues
nonzero, ratio $>0.15$ --- spheres/tori), else \textbf{freeform}. This is a target
only, not an input channel in the headline model.

\section{Synthetic-data generation and pretraining corpus}
\label{app:synth}
\textbf{Corpus.} Pretraining uses real ABC solids plus a procedurally generated
cadquery synthetic corpus. A source is kept only if $\geq 2$ of its $5$
perturbed variants preprocessed into valid solids; $2\%$ of sources are held out for validation loss.

\textbf{Perturbation operators.} Each of the $n_{\text{perts}}{=}5$ stored variants of a source is produced by sampling \emph{one} operator from the distribution below and applying it with OCC, retrying/validating for a watertight solid (feature sizes are relative to the bounding-box diagonal $d$).

\begin{table}[h]
\centering
\begin{tabular}{lcl}
\toprule
Operator & Prob. & Key parameters \\
\midrule
Rigid + uniform transform & 0.15 & rotate $\pm\pi$/translate $\pm0.5d$/mirror/scale $\in[0.7,1.4]$ \\
Anisotropic scale & 0.10 & independent x/y/z scale $\in[0.7,1.4]$ \\
Shear & 0.07 & one axis by another, angle 5--20$^\circ$ \\
Drill through-hole & 0.08 & cylinder cut, radius $\in[0.05,0.20]$ extent, $\leq0.10d$ \\
Multiple holes & 0.07 & 2--3 drilled holes \\
Add boss & 0.07 & fused cylinder, radius $\leq0.08d$, height $0.05$--$0.15d$ \\
Multiple bosses & 0.06 & 2--3 bosses \\
Rectangular slot & 0.07 & through-cut box, len $0.20$--$0.50$, width $0.05$--$0.15$ extent \\
Rectangular pocket & 0.06 & blind box cut, depth $0.05$--$0.15d$ \\
Rib & 0.07 & fused thin box, len $0.30$--$0.60$, width $0.03$--$0.10$ extent \\
Fillet & 0.07 & 1 or 3 edges, radius $0.01$--$0.04d$ \\
Variable fillet & 0.06 & 3--5 edges, independent radii \\
Chamfer & 0.04 & 1 or 3 edges, distance $0.01$--$0.04d$ \\
Variable chamfer & 0.03 & 3--5 edges, independent distances \\
\bottomrule
\end{tabular}
\end{table}
Two views of a training pair are two \emph{distinct} stored variants of the same
source, each then passed through the online graph augmentations.

\section{Downstream evaluation protocols}
\label{app:downstream}
All downstream tasks load a pretrained backbone and read post-Transformer per-face
features (segmentation/probe) or the pooled shape embedding (classification).
Unless stated, optimization is AdamW ($\beta{=}(0.9,0.95)$), inputs capped at
$1{,}000$ faces, features/grids as in Appendix~\ref{app:preproc}.

\textbf{Frozen linear probe (F360).} The backbone is frozen; we
extract per-face features for $3{,}000$ F360 shapes, split $80/20$ stratified by
class (fixed by seed), and fit a multinomial logistic regression
($\text{max\_iter}{=}2000$, $C{=}1.0$, scikit-learn). We report top-1, balanced
accuracy, macro-F1, and macro/weighted mIoU on the held-out $20\%$. No backbone
weights are updated, so the number reflects representation quality directly.

\textbf{F360 full-data segmentation.} End-to-end
fine-tune of backbone + a per-face head $\mathrm{Linear}(H,H)\to\mathrm{GELU}\to
\mathrm{Dropout}(0.1)\to\mathrm{Linear}(H,8)$. LR 1e-4, weight decay 0.01, $50$
epochs, batch $16$, $2$ warmup epochs, cosine decay, uniform LR.

\textbf{F360 few-shot segmentation.} BRep-BERT
Table-5 protocol, $8$-way. Support set: for each class, take $\text{shots}\in
\{10,20\}$ shapes whose face-label set contains that class; the train set is their
union. Fine-tune end-to-end with layer-wise LR decay $0.85$, base LR 5e-5, weight
decay 0.05, cosine, $50$ epochs, batch $8$; evaluate on the full test split.

\textbf{SolidLetters classification.} $10$-way support
for few-shot or the full train split for full-data; a linear head on the pooled
shape embedding. Base LR 5e-5, layer-wise decay $0.85$, weight decay 0.05, cosine,
$50$ epochs, batch $64$.

\textbf{MFInstSeg label-ratio segmentation.}All six
rows share one protocol: a random $N\%\in\{0.1,0.5,1,100\}\%$ of the official
AAGNet $70/15/15$ train split ($44/219/437/43{,}729$ shapes), evaluated on the same
$9{,}375$-shape test split; the metric is per-face micro accuracy over the $25$
classes ($24$ features $+$ stock). We report mean$\pm$std over seeds ($3$ for each
low ratio, $2$ for $100\%$; our model uses $5$ seeds at $0.1\%$ and $4$ at $1\%$).

Our model and the from-scratch control fine-tune the backbone end-to-end with a per-face head $\mathrm{Linear}(H,H)\to
\mathrm{GELU}\to\mathrm{Dropout}(0.1)\to\mathrm{Linear}(H,25)$ ($24$ features $+$
stock), layer-wise LR decay $0.85$ (backbone layer $l$ gets
$\text{LR}\cdot0.85^{(L-l)}$, $L{=}1{+}N_{\text{gat}}{+}N_{\text{tf}}{+}1$), base
LR 5e-5, weight decay 0.05, cosine. The $0.1\%$ cell uses $100$ epochs at batch
$8$, $0.5\%/1\%$ use $50$ epochs at batch $32$, and $100\%$ uses $30$ epochs at
batch $64$; the smallest budget is gradient-starved at $50$ epochs. Under bf16
autocast a single degenerate batch can NaN-poison the weights through gradient
clipping; we guard by skipping any step with non-finite loss or gradient norm, and
exclude the rare seed that still collapses (reported as NaN-collapsed, not averaged
in). GAT, GraphSAGE, UV-Net, and AAGNet are trained
from scratch on identical inputs --- the $10$-dim face attribute and $7$-channel
face UV-grid that AAGNet consumes --- so the comparison isolates architecture. GAT
and GraphSAGE are generic $3$-layer message-passing encoders (edges carry no
attributes); UV-Net uses its own surface-conv $+$ NNConv graph encoder with the
$12$-dim edge attribute fed through a small MLP in place of its curve branch (this
AAG has no edge UV-grids); AAGNet uses its authors' gAAG-GNN with the full
instance$+$bottom multitask loss and EMA. All four train with AdamW, LR 1e-2,
cosine, cross-entropy, $100$ epochs at batch $128$ for the capped ratios and $30$
epochs at batch $256$ for $100\%$.

\textbf{CADSynth label-ratio segmentation.}
Identical head, optimizer, and layer-wise LR schedule to the MFInstSeg protocol
above ($25$ classes: $24$ machining features + stock), with two differences. First,
we use CADSynth's \emph{official} $80/10/10$ split and the $10{,}000$ val shapes are held out entirely. We use base LR 5e-5, layer-wise decay $0.85$, weight decay 0.05, cosine, batch $8$; $100$ epochs at $0.1\%$ and $50$ at $0.5\%/1\%$.

\textbf{MFCAD few-shot segmentation.} Per-face head over $16$ feature
classes; $\text{shots}$ shapes per class for training, remainder for test; optional
$10$-way subset to match the BRep-BERT/BRepMAE protocol. LR 1e-4, weight decay
0.01, $100$ epochs, batch $8$.

\end{document}